\documentclass[preprint,12pt]{elsarticle}
\usepackage{amsmath,amsthm,amssymb,amsfonts,mathdots,graphicx,algorithm,algpseudocode}
\usepackage{a4wide}
\usepackage{csquotes}
\usepackage{bm}
\usepackage{url}
\usepackage{tikz}
\usetikzlibrary{decorations.markings,shapes.geometric,arrows}

\tikzstyle{decision} = [diamond, minimum width=3cm, minimum height=1cm, text centered, draw=black]
\tikzstyle{leaf} = [circle, radius=3pt, draw=black, fill=black]
\tikzstyle{arrow} = [thick,->,>=stealth]

\journal{TBD}

\newcommand{\argmax}{\operatornamewithlimits{argmax}}
\newcommand{\argmin}{\operatornamewithlimits{argmin}}

\newcommand{\defeq}{\stackrel{\text{def}}{=}}
\newcommand{\setR}{\mathbb{R}}
\newcommand{\setN}{\mathbb{N}}

\newtheorem{thm}{Theorem}

\newtheorem{lemma}{Lemma}
\newtheorem{cor}{Corollary}[thm]

\theoremstyle{definition}
\newtheorem{defi}{Definition}

\begin{document}

\begin{frontmatter}
\title{RKL: a general, invariant Bayes solution for Neyman-Scott}
\author[monash]{Michael Brand}
\ead{michael.brand@monash.edu}
\address[monash]{Faculty of IT (Clayton), Monash University, Clayton, VIC 3800, Australia}

\begin{abstract}
Neyman-Scott is a classic example of an estimation problem with a
partially-consistent posterior, for which standard estimation methods tend to
produce inconsistent results. Past attempts to create consistent
estimators for Neyman-Scott have led to ad-hoc solutions, to estimators that
do not satisfy representation invariance, to restrictions over the choice of
prior and more. We present a simple construction for a general-purpose Bayes
estimator, invariant to representation, which satisfies consistency on
Neyman-Scott over any non-degenerate prior.
We argue that the good attributes of the estimator
are due to its intrinsic properties, and generalise beyond Neyman-Scott
as well.
\end{abstract}

\begin{keyword}
Neyman-Scott \sep consistent estimation \sep minEKL \sep Kullback-Leibler \sep Bayes estimation \sep invariance
\end{keyword}

\end{frontmatter}

\section{Introduction}

In \cite{neyman1948consistent}, Neyman and Scott introduced a problem in
consistent estimation that has since been studied extensively in many fields
(see \cite{lancaster2000incidental} for a review). It is known under many names,
such as the problem of partial consistency (e.g., \cite{fan2005semilinear}),
the incidental parameter problem (e.g., \cite{graham2009incidental}),
one way ANOVA (e.g., \cite{lindsay1980nuisance}), the two sample normal
problem (e.g., \cite{ghosh1996noninformative}), or simply as the
Neyman-Scott problem (e.g., \cite{li2003efficiency, kamata2007multilevel}),
each name indicating a slightly different scoping of the problem and a
slightly different emphasis.

In this paper we return to Neyman and Scott's first and most studied example
case of the phenomenon, namely the problem of consistent estimation of
variance based on a fixed number of Gaussian samples.

In Bayesian statistics, this problem has repeatedly been addressed by
analysis over a particular choice of prior (as in \cite{Wallace2005})
or over a particular family of priors (as in \cite{Ghosh1994}).
Priors used include several non-informative priors (see \cite{yang1996catalog}
for a list), including reference priors \cite{berger1992development} and the
Jeffreys prior \cite{Jeffreys1946invariant, Jeffreys1998theory}.

In non-Bayesian statistics, the problem has been addressed by means of
conditional likelihoods, eliminating nuisance parameters by integrating over
them. Analogous techniques exist also in Bayesian analysis.

In \cite[p.~93]{DoweBaxterOliverWallace1998}, Dowe et al.\ opine that
such marginalisation-based solution are apriori unsatisfactory because they
rely on the estimates for individual parameters to not agree with their own
joint estimation. The resulting estimator may therefore be consistent in the
usual sense, but by definition exhibits a form of internal inconsistency.

The paper goes on to present a conjecture by Dowe that elegantly excludes all
such marginalisation-based methods as well as other simple approaches to the
problem by requiring (implicitly) that for a solution to
the Neyman-Scott problem to be satisfactory, it must also satisfy invariance
to representation. This excludes most estimators discussed in the literature
as either inconsistent or invariant. Indeed, Dowe's most modern version of
his conjecture (see \cite[p.~539]{Dowe2008a} and citations within) is that only
estimators belonging to the MML family \cite{wallace1975invariant}
simultaneously satisfy both conditions. The two properties were demonstrated
for two algorithms in the MML family in \cite[p.~201]{Wallace2005} and
\cite{dowe1997resolving}, both using the same reference prior.

More recently, however, \cite{brand2016neymanscott} showed that neither
these two algorithms nor Strict MML \cite{WallaceBoulton1968}
remains consistent under the problem's Jeffreys prior, leading
to the question of whether there is any estimator that retains both
properties in a general setting, i.e.\ under an arbitrary choice of prior.

While we will not answer here the general question of whether an estimation
method can be both consistent and invariant for general estimation problems
(see \cite{BrandXXXX} for a discussion), we describe a novel estimation method,
RKL, that is usable on general point estimation problems, belongs to the
Bayes estimator family, is invariant to representation of both the observation
space and parameter space, and for the Neyman-Scott problem is also consistent
regardless of the choice of prior, whether proper or improper.

The method also satisfies the broader criterion of
\cite{DoweBaxterOliverWallace1998}, in that for the Neyman-Scott problem the
same method can be applied to estimate any subset of the parameters and will
provide the same estimates.

The estimator presented, RKL, is the Bayes estimator whose cost function is
the Reverse Kullback-Leibler divergence. While both the
Kullback-Leibler divergence (KLD) and its reverse (RKL) are well-known and
much-studied functions \cite{cover2012elements} and frequently used in
machine learning contexts (see, e.g., \cite{nowozin2016f}), including for
the purpose of distribution estimation by minimisation of relative
entropy \cite{muhlenbein2005estimation,basu1994minimum}, their usage in the
context of point estimation, where the true distribution is unknown and must
be estimated, is much rarer. Dowe et al.\ \cite{DoweBaxterOliverWallace1998}
introduce the usage of KLD in this context under the name ``minEKL'',
and the use of RKL in the same context is novel to this paper.

We argue that the good consistency properties exhibited by RKL on
Neyman-Scott are not accidental, and describe its advantages for the purposes
of consistency over alternatives such as minEKL on a wider class of problems.

We remark that despite being both invariant and consistent for the problem,
RKL is unrelated to the MML family. It therefore provides further refutation of
Dowe's \cite{Dowe2008a} conjecture.

\section{Definitions}

\subsection{Point estimation}

A point estimation problem \cite{lehmann2006theory} is a set of likelihoods,
$f_{\theta}(x)$, which are
probability density functions over $x\in X$, indexed by a parameter,
$\theta\in\Theta$. Here, $\Theta$ is known as \emph{parameter space}, $X$ as
\emph{observation space}, and $x$ as an \emph{observation}. A point
estimator for such a problem is a function, $\hat{\theta}:X\to\Theta$, matching
each possible observation to a parameter value.

For example, the well-known Maximum Likelihood estimator (MLE), is defined by
\[
\hat{\theta}_{\text{MLE}}(x)\defeq\argmax_{\theta\in\Theta} f_{\theta}(x).
\]
Because this definition is equally applicable for many estimation problems,
simply by substituting in each problem's $f$ and $\Theta$, we say that MLE
is an \emph{estimation method}, rather than just an estimator.

In Bayesian statistics, estimation problems are also
endowed with a prior distribution over their parameter space, denoted by an
everywhere-positive probability density function $h(\theta)$.\footnote{We take
priors that assign a zero probability density to any $\theta$ to be degenerate,
and advocate that in this case such $\theta$ should be excluded from $\Theta$.}
This makes it
possible to think of $f_{\theta}(x)$ and $h(\theta)$ as jointly describing the
joint distribution of a random variable pair
$(\boldsymbol{\theta},\boldsymbol{x})$, where $h$ is the marginal distribution
of $\boldsymbol{\theta}$ and $f_{\theta}$ is the conditional distribution
of $\boldsymbol{x}$ given $\boldsymbol{\theta}=\theta$. We therefore use
$f(x|\theta)$ as a synonym for $f_{\theta}(x)$.

It is convenient to generalise the idea of a prior distribution by
allowing priors to be \emph{improper}, in the sense that
\[
\int_{\Theta} h(\theta) \text{d}\theta = \infty.
\]
This means that $(\boldsymbol{\theta},\boldsymbol{x})$ is no longer described
by a probability distribution, but rather by a general measure.
When choosing an improper prior more care must be taken: for a prior to be
valid, it should still be possible to compute the (equally improper) marginal
distribution of $x$ by means of
\[
r(x)\defeq\int_{\Theta} f(x|\theta) h(\theta) \text{d}\theta,
\]
as without this Bayesian calculations quickly break down. We will, throughout,
assume all priors to be valid.

Where $r(x)\ne 0$, we can also define the posterior distribution
\[
f(\theta|x) = \frac{f(x|\theta) h(\theta)}{r(x)}.
\]
Note that even when $h$ and $r$ are improper, $f(x|\theta)$ and $f(\theta|x)$
are both proper probability distribution functions.

\begin{lemma}\label{L:rpos}
In any estimation problem where $f(x|\theta)$ is always positive, $r(x)\ne 0$
for every $x$.
\end{lemma}

\begin{proof}
Fix $x$, and for any natural $i$ let
$\Theta_i=\{\theta\in\Theta|\lceil 1/f(x|\theta) \rceil=i\}$.

The sequence $\{\Theta_i\}_{i\in\setN}$ partitions
$\Theta$ into a countable number of parts. As a result, at least one such part
has a positive prior probability.
\[
\int_{\Theta_i} h(\theta) \text{d}\theta>0.
\]

We can now bound $r(x)$ from below by
\begin{align*}
r(x)&=\int_{\Theta} f(x|\theta) h(\theta) \text{d}\theta \\
&\ge \int_{\Theta_i} f(x|\theta) h(\theta) \text{d}\theta \\
&\ge \int_{\Theta_i} (1/i) h(\theta) \text{d}\theta \\
&= (1/i) \int_{\Theta_i} h(\theta) \text{d}\theta \\
&>0.
\end{align*}
\end{proof}

In this paper, we will throughout be discussing estimation problems where
the conditions of Lemma~\ref{L:rpos} hold, for which reason we will always
assume that $r(x)$ is positive. Coupled with the fact that $h(\theta)$ is,
by assumption, also always positive, this leads to positive, well defined,
$f(x|\theta)$, positive $f(\theta|x)$ and positive $f(x|\theta) h(\theta)$.

\subsection{Consistency}

In defining point estimation, we treated $x$ as a single variable. Typically,
however, $x$ is a vector. Consider, for example, an observation space
$X=X_1 \times X_2 \times \cdots$. In this case, the observation takes the form
$x=(x_1,x_2,\ldots)$.

Typically, every $f_\theta$ in an estimation problem is defined such that
individual $x_n$ are independent and identically distributed, but
we will not require this.

For estimation methods that can estimate $\theta$ from every prefix
$x_{1:N}=(x_1,\ldots,x_N)$, it is possible to define \emph{consistency}, which
is one desirable property for an estimation problem, as follows
\cite{lehmann2006theory}.

\begin{defi}[Consistency]
Let $P$ be an estimation problem over observation space
$X=X_1 \times X_2 \times \cdots$, and let $\{P_N\}_{N\in\setN}$ be the
sequence of estimation problems created by
taking only $x_{1:N}=(x_1,\ldots,x_N)$ as the observation.

An estimation method $\hat{\theta}$ is said to be \emph{consistent} on $P$ if
for every $\tilde{\theta}\in\Theta$ and every neighbourhood
$S$ of $\tilde{\theta}$, if $x$ is taken from the distribution
$f_{\tilde{\theta}}$ then almost surely
\[
\lim_{N\to\infty} \hat{\theta}(x_{1:N})\in S,
\]
where the choice of estimation problem for $\hat{\theta}$ is understood from
the choice of parameter.
\end{defi}

\subsection{The Neyman-Scott problem}

\begin{defi}
The \emph{Neyman-Scott problem} \citep{neyman1948consistent}
is the problem of jointly estimating the tuple
$(\sigma^2,\mu_1,\ldots,\mu_N)$ after observing
$(x_{nj} : n=1,\ldots,N; j=1,\ldots,J)$, each element of which is independently
distributed $x_{nj} \sim \mathcal{N}(\mu_n,\sigma^2)$.

It is assumed that $J\ge 2$, and for brevity we take $\mu$ to be the vector
$(\mu_1,\ldots,\mu_N)$.
\end{defi}

The Neyman-Scott problem is a classic case-study for consistency due to its
partially-consistent posterior.

Loosely speaking, a posterior, i.e.\ the distribution of $\boldsymbol{\theta}$ given the
observations, is called \emph{inconsistent} if even
in the limit, as $N\to\infty$, there is no $\hat{\theta}$ such that every
neighbourhood of $\hat{\theta}$ tends to total probability $1$.
(See \citep{ghosal1997review} for a formal definition.)
In such a case it
is clear that no estimation method can be consistent. When keeping $J$
constant and taking $N$ to infinity, the Neyman-Scott problem creates such an
inconsistent posterior, because the uncertainty in the distribution of each
$\mu_n$ remains high.

The problem is, however, \emph{partially consistent} in that the posterior
distribution for $\sigma^2$ does converge, so it is possible for an estimation
method to estimate it, individually, in a consistent way.

For example, the estimator
\begin{equation}\label{Eq:var_est}
\hat{\sigma}^2(x)=\frac{J}{J-1}s^2.
\end{equation}
is a well-known consistent estimator for $\sigma^2$, where
\[
s^2 \defeq \frac{\sum_{n=1}^{N} \sum_{j=1}^{J} (x_{nj}-m_n)^2}{NJ}
\]
and
\[
m_n \defeq \frac{\sum_{j=1}^{J} x_{nj}}{J}.
\]
(We use $m$ to denote the vector $(m_1,\ldots,m_N)$.)

The interesting question for Neyman-Scott is what estimation methods
can be devised for the joint estimation problem, such that their estimate
for $\sigma^2$, as part of the larger estimation problem, is consistent.

Famously, MLE's estimate for $\sigma^2$ is in this scenario $s^2$, which is
not consistent, and the same is true for the estimates of many other popular
estimation methods such as Maximum Aposteriori Probability (MAP) and
Minimum Expected Kullback-Leibler Distance (minEKL).

It is, of course, possible for an estimation method to work on each coordinate
independently. An example of an estimation method that does this is
Posterior Expectation (PostEx). Such methods, however, rely on a particular
choice of
description for the parameter space (and sometimes also for the observation
space). If one were to estimate $\sigma$, for example, instead of $\sigma^2$,
the estimates of PostEx for the same estimation problem would change
substantially. PostEx may therefore be consistent for the problem,
but it is not \emph{invariant} to representation of $X$ and $\Theta$.

The question therefore arises whether it is possible to construct an estimation
method that is both invariant (like MLE) and consistent (like the estimators
of \cite{Ghosh1994}), and that, moreover, unlike the estimators of
\cite{dowe1997resolving, Wallace2005}, retains these properties for all possible
priors.

Typically, priors studied in the literature can be described as
$1/\sigma^{F(N)}$ for some function $F$. These are priors where $\mu_n$
values are independent and uniform given $\sigma$. The studied methods often
break down, as in the case of \cite{dowe1997resolving, Wallace2005}, simply
by switching to another $F$.

The RKL estimator introduced here, however, remains consistent under
extremely general priors, including ones with
$\mu$ distributions that, even given $\sigma$, are not uniform, not
identically distributed, and not independent.

\section{The RKL estimator}

\begin{defi}
The \emph{Reverse Kullback-Leibler (RKL) estimator} is a Bayes estimator,
i.e.\ it is an estimator that
can be defined as a minimiser of the conditional expectation of a loss function,
$L$.
\[
\hat{\theta}(x)=\argmin_{\theta'\in\Theta} \int_{\Theta} L(\theta,\theta') f(\theta|x) \text{d}\theta,
\]
where for RKL the $L$ function is defined by
\[
L(\theta,\theta')\defeq D_{\text{KL}}(f_{\theta'}\|f_{\theta}).
\]

Here, $D_{\text{KL}}(f\|g)$ is the
Kullback-Leibler divergence (KLD) \cite{kullback1997information} from $g$ to
$f$,
\[
D_{\text{KL}}(f\|g)=\int_X\log\left(\frac{f(x)}{g(x)}\right)f(x)\text{d}x.
\]
Equivalently, it is the entropy of $f$ relative to $g$.
\end{defi}

This definition looks quite similar to the definition of the standard minEKL
estimator, which uses the Kullback-Leibler divergence as its loss function,
except that the parameter order has been reversed. Instead of utilising
$D_{\text{KL}}(f_{\theta}\|f_{\theta'})$, as in the
original definition of minEKL, we use
$D_{\text{KL}}(f_{\theta'}\|f_{\theta})$.
Because the Kullback-Leibler divergence is non-symmetric, the result is a
different estimator.

Although the Reverse Kullback-Leibler divergence is a well-known
$f$-divergence \cite{liese2006divergences, nowozin2016f}, it has to our
knowledge never been applied as a loss function in Bayes estimation.

\section{Invariance and consistency}

In terms of invariance to representation, it is clear that RKL inherits
the good properties of $f$-divergences.

\begin{lemma}
RKL is invariant to representation of $\Theta$ and of $X$.
\end{lemma}

\begin{proof}
The RKL loss function is dependent only on distributions of $\boldsymbol{x}$
given a choice of $\theta$. Renaming the $\theta$ therefore does not affect
it. Furthermore, the loss function is an $f$-divergence, and therefore
invariant to reparameterisations of $\boldsymbol{x}$ \cite{qiao2008f}.
\end{proof}

More interesting is the analysis of RKL's consistency. In this section, we
analyse RKL's consistency on Neyman-Scott. In the next section, we turn to
its consistency properties in more general settings.

\begin{thm}\label{T:main}
RKL is consistent for Neyman-Scott over any valid, non-degenerate prior.
\end{thm}

To begin, let us describe the estimator more concretely.

\begin{lemma}\label{L:esigma}
For Neyman-Scott over any valid, non-degenerate prior,
\[
\hat{\sigma}^2_\text{RKL}(x)=\text{E}^{-1}\left(\frac{1}{\boldsymbol{\sigma}^2}\middle|x\right).
\]
\end{lemma}

\begin{proof}
In Neyman-Scott, each observation $x_{nj}$ is distributed independently
with some variance $\sigma^2$ and some mean $\mu_n$,
\[
f^{nj}_{\sigma^2,\mu}(x_{nj})=f(x_{nj}|\sigma^2,\mu)=\frac{1}{\sqrt{2\pi\sigma^2}}e^{-\frac{(x_{nj}-\mu_n)^2}{2\sigma^2}}.
\]

The KLD between two such distributions is
\begin{align*}
D_{\text{KL}}&(f^{nj}_{\tilde{\sigma}^2,\tilde{\mu}}\|f^{nj}_{\sigma'^2,\mu'})
=\int_\mathbb{R} \frac{1}{\sqrt{2\pi\tilde{\sigma}^2}}e^{-\frac{(x_{nj}-\tilde{\mu}_n)^2}{2\tilde{\sigma}^2}} \log\left(\frac{\frac{1}{\sqrt{2\pi\tilde{\sigma}^2}}e^{-\frac{(x_{nj}-\tilde{\mu}_n)^2}{ 2\tilde{\sigma}^2}}}{\frac{1}{\sqrt{2\pi\sigma'^2}}e^{-\frac{(x_{nj}-\mu'_n)^2}{ 2\sigma'^2}}}\right) \text{d}x_{nj} \\
&=\frac{1}{2}\left[\frac{\tilde{\sigma}^2}{\sigma'^2}-1-\log\left(\frac{\tilde{\sigma}^2}{\sigma'^2}\right)\right]+\frac{1}{2\sigma'^2}(\tilde{\mu}_n-\mu'_n)^2.
\end{align*}

Given that these observations are independent, the KLD over all observations
is the sum of the KLD over the individual observations:
\begin{align*}
L((\sigma'^2,\mu'),(\tilde{\sigma}^2,\tilde{\mu}))
&=D_{\text{KL}}(f_{\tilde{\sigma}^2,\tilde{\mu}}\|f_{\sigma'^2,\mu'}) \\
&=\frac{NJ}{2}\left[\frac{\tilde{\sigma}^2}{\sigma'^2}-1-\log\left(\frac{\tilde{\sigma}^2}{\sigma'^2}\right)\right]+\frac{J}{2\sigma'^2}|\tilde{\mu}-\mu'|^2.
\end{align*}

The Bayes risk associated with choosing a particular
$(\tilde{\sigma}^2,\tilde{\mu})$ as the estimate is therefore
\[
\int_\Theta f(\sigma'^2,\mu'|x) \left(\frac{NJ}{2}\left[\frac{\tilde{\sigma}^2}{\sigma'^2}-1-\log\left(\frac{\tilde{\sigma}^2}{\sigma'^2}\right)\right]+\frac{J}{2\sigma'^2}|\tilde{\mu}-\mu'|^2\right) \text{d}(\sigma'^2,\mu').
\]

In finding the $(\tilde{\sigma}^2,\tilde{\mu})$ combination that minimises
this risk, it is clear that the choice of $\tilde{\sigma}$ and of each
$\tilde{\mu}_n$ can be made separately, as the expression can be split into
additive components, each of which is dependent only on one variable.

The risk component associated with each $\tilde{\mu}_n$ is
\begin{align}\label{Eq:R_u_n}
R(\tilde{\mu}_n)
&=\int_\Theta f(\sigma'^2,\mu'|x) \frac{J}{2\sigma'^2}(\tilde{\mu}_n-\mu'_n)^2\text{d}(\sigma'^2,\mu') \\
&=\int_{\mathbb{R}^+} \int_{\mathbb{R}} f(\sigma'^2,\mu'_n|x) \frac{J}{2\sigma'^2}(\tilde{\mu}_n-\mu'_n)^2\text{d}\mu'_n \text{d}\sigma'^2.
\end{align}

More interestingly in the context of consistency, the risk component
associated with $\tilde{\sigma}^2$ is
\begin{align*}
R(\tilde{\sigma}^2)
&=\int_\Theta f(\sigma'^2,\mu'|x)\frac{NJ}{2}\left[\frac{\tilde{\sigma}^2}{\sigma'^2}-1-\log\left(\frac{\tilde{\sigma}^2}{\sigma'^2}\right)\right] \text{d}(\sigma'^2,\mu') \\
&=\int_{\mathbb{R}^+} f(\sigma'^2|x)\frac{NJ}{2}\left[\frac{\tilde{\sigma}^2}{\sigma'^2}-1-\log\left(\frac{\tilde{\sigma}^2}{\sigma'^2}\right)\right] \text{d}\sigma'^2.
\end{align*}

% Up to multiplicative constants, this equals
% \[
% \int_{\mathbb{R}^+} f(\sigma'^2|x)\left[\frac{\tilde{\sigma}^2}{\sigma'^2}-1-\log\left(\frac{\tilde{\sigma}^2}{\sigma'^2}\right)\right] \text{d}\sigma'^2.
% \]

This expression is a Bayes risk for the one-dimensional problem of estimating
$\sigma^2$ from $x$ (with a specific loss function),
a type of problem that is typically not difficult for Bayes estimators.

We will utilise the fact that the risk function is a linear combination
of the $L$ functions, when taking these as functions of
$(\tilde{\sigma}^2,\tilde{\mu})$, indexed by $(\sigma'^2,\mu')$, and these
functions, both in their complete form and when separated to components,
are convex, differentiable functions with a unique minimum. We conclude that the
risk function is also a convex function, and that its minimum can be found by
taking its derivative to zero, which, in turn, is a linear combination of the
derivatives of the individual loss functions. To solve for $\hat{\sigma}$,
for example, we therefore solve the equation
\[
\frac{\text{d}\int_{\mathbb{R}^+} f(\sigma'^2|x) \frac{NJ}{2}\left[\frac{\hat{\sigma}^2}{\sigma'^2}-1-\log\left(\frac{\hat{\sigma}^2}{\sigma'^2}\right)\right] \text{d}\sigma'^2}{\text{d}\hat{\sigma}^2}=0.
\]

This leads to
\[
\int_{\mathbb{R}^+} f(\sigma'^2|x)\left[\frac{1}{\sigma'^2}-\frac{1}{\hat{\sigma}^2}\right] \text{d}\sigma'^2=0.
\]
\[
\text{E}\left(\frac{1}{\boldsymbol{\sigma}^2}\middle|x\right)-\frac{1}{\hat{\sigma}^2}=0.
\]

So for the Neyman-Scott problem, the $\sigma$ portion of the RKL estimator is
\[
\hat{\sigma}^2(x)=\text{E}^{-1}\left(\frac{1}{\boldsymbol{\sigma}^2}\middle|x\right),
\]
as required.
\end{proof}

For completion, we remark that using the same analysis on \eqref{Eq:R_u_n} we
can determine that the estimate for each $\mu_n$ is
\[
\hat{\mu}_n(x)=\hat{\sigma}^2(x)\text{E}\left(\frac{\boldsymbol{\mu_n}}{\boldsymbol{\sigma^2}}\middle|x\right).
\]

We now turn to the question of how consistent this estimator is
for $\sigma^2$.

\begin{proof}[Proof of Theorem~\ref{T:main}]

We want to calculate
\begin{align}\label{Eq:postex}
\begin{split}
\text{E}\left(\frac{1}{\boldsymbol{\sigma}^2}\middle|x\right)
&=\int_0^{\infty} \frac{1}{\sigma^2} f(\sigma|x) \text{d}\sigma \\
&=\int_\Theta \frac{1}{\sigma^2} f(\sigma,\mu|x) \text{d}(\sigma,\mu) \\
&=\frac{\int_\Theta \frac{1}{\sigma^2} f(x|\sigma,\mu) h(\sigma,\mu) \text{d}(\sigma,\mu)}{\int_\Theta f(x|\sigma,\mu) h(\sigma,\mu) \text{d}(\sigma,\mu)}.
\end{split}
\end{align}

The fact that Neyman-Scott has \emph{any} consistent estimators (such as,
for example, the one presented in \eqref{Eq:var_est}) indicates that the
posterior probability, given $x$, for $\boldsymbol{\sigma}$ to be outside
any neighbourhood of the real $\tilde{\sigma}$ tends to zero. For this reason,
posterior expectation in any case where the estimation variable is bounded,
is necessarily consistent.

Here this is not the case, because $1/\sigma^2$ tends to infinity when
$\sigma$ tends to zero. However, given that the denominator of \eqref{Eq:postex}
is precisely the marginal $r(x)$, and therefore by assumption positive, it is
enough to show that with probability $1$,
\[
\lim_{N\to\infty}\int_{(\sigma',\mu')\in\Theta:\sigma'<\alpha} \frac{1}{\sigma^2} f(x_{11},\ldots, x_{NJ}|\sigma,\mu) h(\sigma,\mu) \text{d}(\sigma,\mu) = 0,
\]
for some $\alpha>0$, to prove that all parameter options with $\sigma<\alpha$
cannot affect the expectation, and that the resulting estimator is
therefore consistent.

To do this, the first step is to recall that by assumption $r(x)$ is finite,
including at $N=1$, and therefore, for any $\alpha$,
\[
r_{\alpha}(x)\defeq\int_{(\sigma',\mu')\in\Theta: \sigma'<\alpha} f(x_{11},\ldots, x_{1J}|\sigma,\mu) h(\sigma,\mu) \text{d}(\sigma,\mu) < \infty.
\]

Let us now define
\begin{align*}
g_{\sigma,\mu,N}(x)&\defeq\frac{\frac{1}{\sigma^2} f(x_{11},\ldots, x_{NJ}|\sigma,\mu)}{f(x_{11},\ldots, x_{1J}|\sigma,\mu)} \\
&=\frac{\frac{1}{\sigma^2}\frac{1}{(2\pi\sigma^2)^{NJ/2}} e^{-\frac{1}{2\sigma^2}\left(NJs^2+J\sum_{n=1}^{N}(m_n-\mu_n)^2\right)}}{\frac{1}{(2\pi\sigma^2)^{J/2}} e^{-\frac{1}{2\sigma^2}\left(Js^2+J(m_1-\mu_1)^2\right)}} \\
&=\frac{1}{\sigma^2}\frac{1}{(2\pi\sigma^2)^{(N-1)J/2}} e^{-\frac{1}{2\sigma^2}\left((N-1)Js^2+J\sum_{n=2}^{N}(m_n-\mu_n)^2\right)}.
\end{align*}

Differentiating $g_{\sigma,\mu,N}(x)$ by $\sigma$, we conclude that this is a
monotone increasing function for every
\[
\sigma^2<\frac{(N-1)J}{(N-1)J+2} s^2 + \frac{J}{(N-1)J+2}\sum_{n=2}^{N}(m_n-\mu_n)^2,
\]
and so in particular for any $\sigma^2<s^2$ for a sufficiently large $N$.

For a given $\sigma$, $N$ and $x$, $g_{\sigma,\mu,N}(x)$ reaches its maximum
at $m=\mu$.

Furthermore, note that for $\sigma^2<\frac{1}{2\pi}$, $g_{\sigma,\mu,N}(x)$
strictly decreases with $N$, and tends to zero.

Let us now choose an $\alpha$ value smaller than both $s$ and $(2\pi)^{-1/2}$,
and calculate
\begin{align*}
\lim_{N\to\infty}&\int_{(\sigma',\mu')\in\Theta: \sigma'<\alpha} \frac{1}{\sigma^2} f(x_{11},\ldots, x_{NJ}|\sigma,\mu) h(\sigma,\mu) \text{d}(\sigma,\mu) \\
&= \lim_{N\to\infty}\int_{(\sigma',\mu')\in\Theta: \sigma'<\alpha} g_{\sigma,\mu,N}(x) f(x_{11},\ldots, x_{1J}|\sigma,\mu) h(\sigma,\mu) \text{d}(\sigma,\mu) \\
&\le\lim_{N\to\infty} \left(\max_{(\sigma',\mu')\in\Theta: \sigma'<\alpha} g_{\sigma,\mu,N}(x)\right)\int_{(\sigma',\mu')\in\Theta: \sigma'<\alpha}f(x_{11},\ldots, x_{1J}|\sigma,\mu) h(\sigma,\mu) \text{d}(\sigma,\mu) \\
&= \lim_{N\to\infty} g_{\alpha,m,N}(x) r_\alpha(x) \\
&= 0.
\end{align*}
\end{proof}

\section{General consistency of RKL}

The results above may seem unintuitive: in designing a good loss function,
$L(\theta,\theta')$, for Bayes estimators, one strives to find a measure that
reflects how bad it would be to return the estimate $\theta'$ when the
correct value is $\theta$, and yet the RKL loss function, contrary to minEKL,
seems to be using $\theta'$ as its baseline, and measures how different the
distribution induced by $\theta$ would be at every point. Why should this
result in a better metric?

To answer, consider the portion of the (forward) Kullback-Leibler divergence
that depends on $\mu$. In the one dimensional case, and when calculating the
divergence between two normal distributions $\mathcal{N}(\mu_1,\sigma_1^2)$ and
$\mathcal{N}(\mu_2,\sigma_2^2)$, this is
$\frac{(\mu_1-\mu_2)^2}{\sigma_2^2}$.

The difference between the two expectations is measured in terms of how many
$\sigma_2$ standard deviations away the two are.

Because $\sigma_2$, the yardstick for the divergence of the expectations, is
used in minEKL as $\hat{\sigma}$, a value to be estimated, it is possible to
reduce the measured divergence by unnecessarily inflating the estimate.

By contrast, RKL uses the true $\sigma$ value as its yardstick, for which
reason no such inflation is possible for it. This makes its estimates
consistent.

This is a general trait of RKL, in the sense that it uses as its loss metric
the entropy of the \emph{estimate} relative to the \emph{true value}, even
though the true value is unknown.

While not a full-proof method of avoiding problems created by
partial consistency (or, indeed, even some fully consistent scenarios),
this does address the problem in a range of situations.

The following alternate characterisation of RKL gives better intuition
regarding when the method's estimates are consistent.

\begin{defi}[RKL reference distribution]
Given an estimation problem $(\boldsymbol{\theta},\boldsymbol{x})$ with
likelihoods $f(x|\theta)$, let $g_x:X\to\setR^+$ be defined by
\begin{equation}\label{Eq:ref}
g_x(y)=e^{\text{E}(\log(f(y|\boldsymbol{\theta}))|x)}.
\end{equation}

If for every $x$, $G_x=\int_X g_x(y) \text{d}y$ is defined and nonzero,
let $\hat{g}_x$, the \emph{RKL reference distribution}, be the probability
density function $\hat{g}_x(y)=g_x(y)/G_x$, i.e.\ the normalised version
of $g_x$.
\end{defi}

\begin{thm}\label{T:elan}
Let $(\boldsymbol{\theta},\boldsymbol{x})$ be an estimation problem with
likelihoods $f_{\theta}$ and RKL reference distribution $\hat{g}_x$.

\begin{equation}\label{Eq:elan}
\hat{\theta}_{\text{RKL}}(x)=
\argmin_{\theta'\in\Theta} D_{\text{KL}}(f_{\theta'}\|\hat{g}_x).
\end{equation}
\end{thm}

\begin{proof}
Expanding the RKL formula, we get
\begin{align*}
\hat{\theta}_{\text{RKL}}(x)&=\argmin_{\theta'\in\Theta} \int_{\Theta} L(\theta,\theta') f(\theta|x) \text{d}\theta \\
&=\argmin_{\theta'\in\Theta} \int_{\Theta} \int_X\log\frac{f(y|\theta')}{f(y|\theta)}f(y|\theta')\text{d}y f(\theta|x) \text{d}\theta \\
&=\argmin_{\theta'\in\Theta} \int_{\Theta} \int_X\left(\log(f(y|\theta'))-\log(f(y|\theta))\right) f(y|\theta')\text{d}y f(\theta|x) \text{d}\theta \\
&=\argmin_{\theta'\in\Theta} \int_X\left(\log(f(y|\theta'))-\int_{\Theta}\log(f(y|\theta)) f(\theta|x) \text{d}\theta\right) f(y|\theta')\text{d}y \\
&=\argmin_{\theta'\in\Theta} \int_X\left(\log(f(y|\theta'))-\text{E}(\log(f(y|\boldsymbol{\theta}))|x)\right) f(y|\theta')\text{d}y \\
&=\argmin_{\theta'\in\Theta} \int_X\log\frac{f(y|\theta')}{g_x(y)} f(y|\theta')\text{d}y \\
&=\argmin_{\theta'\in\Theta} \int_X\log\frac{f(y|\theta')}{\hat{g}_x(y)} f(y|\theta')\text{d}y \\
&=\argmin_{\theta'\in\Theta} D_{\text{KL}}(f_{\theta'}\|\hat{g}_x).
\end{align*}

The move from $g_x$ to $\hat{g}_x$ is justified because the difference is a
positive multiplicative constant $G_x$, translating to an additive constant
after the $\log$, and therefore not altering the result of the $\argmin$.
% 
% For completion, we note that $G_x$ is in the range $0<G_x\le 1$. It cannot
% be greater than one, because of convexity: for any function $u(\theta)$ and
% any distribution of $\boldsymbol{\theta}$,
% \[
% \text{E}(\log(u(\boldsymbol{\theta})))\le\log\text{E}(u(\boldsymbol{\theta})),
% \]
% and therefore also
% \begin{align*}
% \int_X \text{E}(\log(u_y(\boldsymbol{\theta}))|x)\text{d}y&\le\log\int_X\text{E}(u_y(\boldsymbol{\theta})|x)\text{d}y \\
% &=\log\text{E}\left(\int_X u_y(\boldsymbol{\theta})\text{d}y|x\right).
% \end{align*}
% 
% Substituting $u_y(\theta)=f(y|\theta)$, we get that
% $\int_X u_y(\boldsymbol{\theta})\text{d}y=1$, regardless of the distribution
% of $\theta$, so
% \[
% e^{\int_X \text{E}(\log(u_y(\boldsymbol{\theta}))|x)\text{d}y}\le 1,
% \]
% as desired.
\end{proof}

This alternate characterisation makes RKL's behaviour on Neyman-Scott
more intuitive: the RKL reference distribution is calculated using
an expectation over log-scaled likelihoods. In the case of Gaussian
distributions, representing the likelihoods in log scale results in
parabolas. Calculating the expectation over parabolas leads to a parabola
whose leading coefficient is the expectation of the leading coefficients
of the original parabolas. This directly justifies Lemma~\ref{L:esigma}
and can easily be extended also to multivariate normal distributions.

Furthermore, the alternate characterisation provides a more general
sufficient condition for RKL's consistency in cases where the posterior is
consistent.

\begin{defi}[Distinctive likelihoods]
An estimation problem $(\boldsymbol{\theta},\boldsymbol{x})$ with likelihoods
$f_\theta$ is said to have \emph{distinctive likelihoods} if for any sequence
$\{\theta_i\}_{i\in\setN}$ and any $\theta$,
\[
f_{\theta_i}\xrightarrow{\text{TV}} f_{\theta} \Rightarrow \theta_i\to\theta,
\]
where ``$\xrightarrow{\text{TV}}$'' indicates total variation distance.
\end{defi}

\begin{cor}
If $(\boldsymbol{\theta},\boldsymbol{x})$ is an estimation problem with
distinctive likelihoods $f_\theta$ and an RKL reference distribution
$\hat{g}_x$, such that for every $\tilde{\theta}$, with
probability $1$ over an $x=(x_1,\ldots,)$ generated from the distribution
$f_{\tilde{\theta}}$,
\begin{equation}\label{Eq:limit}
\lim_{N\to\infty} D_{\text{KL}}(f_{\tilde{\theta}}\|\hat{g}_{(x_1,\ldots,x_N)})=0,
\end{equation}
then RKL is consistent on the problem.
\end{cor}

\begin{proof}
By Theorem~\ref{T:elan},
\[
\hat{\theta}_{\text{RKL}}(x)=
\argmin_{\theta'\in\Theta} D_{\text{KL}}(f_{\theta'}\|\hat{g}_x).
\]
However, \eqref{Eq:limit} gives an upper bound on the minimum divergence, so
\[
\lim_{N\to\infty} \min_{\theta'\in\Theta} D_{\text{KL}}(f_{\theta'}\|\hat{g}_{(x_1,\ldots,x_N)})=0.
\]

By Pinsker's inequality \cite{cover2012elements}, $\hat{g}_{(x_1,\ldots,x_N)}$
converges under the total variations metric to both $f_{\tilde{\theta}}$ and
$f_{\hat{\theta}_{\text{RKL}}(x_1,\ldots,x_N)}$. By the triangle inequality,
the total variation distance between $f_{\tilde{\theta}}$ and
$f_{\hat{\theta}_{\text{RKL}}(x_1,\ldots,x_N)}$ tends to zero, and therefore
by assumption of likelihood distinctiveness
\[
\lim_{N\to\infty} \hat{\theta}_{\text{RKL}}(x_1,\ldots,x_N)=\tilde{\theta}.
\]
\end{proof}

RKL's consistency is therefore guaranteed in cases where $\hat{g}_x$
converges to $f_{\tilde{\theta}}$. This type of guarantee is similar to
guarantees that exist also for other estimators, such as minEKL and posterior
expectation, in that convergence of the estimator is reliant on the convergence
of a particular expectation function. Having a consistent posterior guarantees
that all $\theta$ values outside any neighbourhood of $\tilde{\theta}$
receive a posterior probability density tending to zero, but when calculating
expectations such probability densities are multiplied by the random
variable over which the expectation is calculated, for which reason if it
tends to infinity fast enough compared to the speed in which the probability
density tends to zero, the expectation may not converge.

RKL's distinction over posterior expectation and minEKL, however, is that,
as demonstrated in \eqref{Eq:ref}, the random variable of the
expectation is taken in log scale, making it much harder for its magnitude to
tend quickly to infinity.

This makes RKL's consistency more robust than minEKL's over a large class of
realistic estimation problems.

\section{Conclusions and future research}

We've introduced RKL as a novel, simple, general-purpose,
parameterisation-invariant Bayes estimation method, and showed it to be
consistent over a large class of estimation problems with consistent
posteriors and over Neyman-Scott one-way ANOVA problems regardless of one's
choice of prior.

Beyond being an interesting and useful new estimator in its own right
and a satisfactory
solution to the Neyman-Scott problem, the estimator also serves as a
direct refutation to Dowe's conjecture in \cite[p.~539]{Dowe2008a}.

The robustness of RKL's consistency was traced back to its reference
distribution being calculated as an expectation in log-scale.

This leaves open the question of whether there are other types of scaling
functions, with even better properties, that can be used instead of
log-scale, without losing the estimator's invariance to parameterisation.

\end{document}